\documentclass{article}

    \PassOptionsToPackage{numbers, compress}{natbib}

    \usepackage[preprint]{neurips_2020}

\usepackage[utf8]{inputenc} 
\usepackage[T1]{fontenc}    
\usepackage{hyperref}       
\usepackage{url}            
\usepackage{booktabs}       
\usepackage{amsfonts}       
\usepackage{nicefrac}       
\usepackage{microtype}      

\usepackage{microtype}
\usepackage{graphicx}
\usepackage{subcaption, float}
\usepackage{booktabs} 
\usepackage{xcolor}

\usepackage{amsmath,amsfonts,bm}

\def\eqref#1{equation~\ref{#1}}

\def\1{\bm{1}}

\def\vtheta{{\bm{\theta}}}
\def\vphi{{\bm{\phi}}}

\def\vc{{\bm{c}}}
\def\vd{{\bm{d}}}

\def\vf{{\bm{f}}}
\def\vg{{\bm{g}}}
\def\vh{{\bm{h}}}

\def\vq{{\bm{q}}}

\def\vv{{\bm{v}}}
\def\vw{{\bm{w}}}

\def\vy{{\bm{y}}}
\def\vz{{\bm{z}}}

\def\mM{{\bm{M}}}

\DeclareMathAlphabet{\mathsfit}{\encodingdefault}{\sfdefault}{m}{sl}
\SetMathAlphabet{\mathsfit}{bold}{\encodingdefault}{\sfdefault}{bx}{n}

\newcommand{\E}{\mathbb{E}}

\usepackage[ruled,vlined,linesnumbered,noresetcount]{algorithm2e}

\usepackage{amsthm}
\newtheorem{theorem}{Theorem}
\newtheorem{remark}{Remark}
\newtheorem{lemma}{Lemma}
\newtheorem{corollary}{Corollary}

\def\lemref#1{Lemma~\ref{#1}}
\def\thmref#1{Theorem~\ref{#1}}

\title{FedGAN: Federated Generative Adversarial Networks for Distributed Data}

\author{%
  Mohammad Rasouli, Tao Sun, Ram Rajagopal, \\
  Stanford University, Stanford, CA,  94305 \\
  \texttt{\{rasoulim, luke18, ramr\}@stanford.edu} \\
}

\begin{document}

\maketitle

\begin{abstract}
  We propose Federated Generative Adversarial Network (FedGAN) for training a GAN across distributed sources of non-independent-and-identically-distributed data sources subject to communication and privacy constraints. Our algorithm uses local generators and discriminators which are periodically synced via an intermediary that averages and broadcasts the generator and discriminator parameters. We theoretically prove the convergence of FedGAN with both equal and two time-scale updates of generator and discriminator, under standard assumptions, using stochastic approximations and communication efficient stochastic gradient descents. We experiment FedGAN on toy examples (2D system, mixed Gaussian, and Swiss role), image datasets (MNIST, CIFAR-10, and CelebA), and time series datasets (household electricity consumption and electric vehicle charging sessions). We show FedGAN converges and has similar performance to general distributed GAN, while reduces communication complexity. We also show its robustness to reduced communications.
\end{abstract}

\section{Introduction}
\label{intro}
Generative adversarial network (GAN) is proposed by \citep{goodfellow2014generative} for generating fake data similar to the original data and has found wide applications. In lots of cases data is distributed across multiple sources, with data in each source being too limited in size and diversity to locally train an accurate GAN for the entire population of distributed data. On the other hand, due to privacy constraints data can not be shared or pooled centrally. Therefore, a distributed GAN algorithm is required for training a GAN representing the entire population; such distributed GAN also allows generating publicly accessible data \citep{yonetani2019decentralized}. Current distributed GAN algorithms require large communication bandwidth among data sources or between data sources and an intermediary (to ensure convergence) due to architectures that separate generators from discriminators \citep{augenstein2019generative,hardy2018md}. But in many applications communication bandwidth is limited, e.g. energy, mobile communications, finance, and sales \citep{yang2019federated}. 
Communication-efficient distributed GAN is an open problem. We propose an architecture which places local discriminators with local generators, synced occasionally through an intermediary.

Communication-efficient learning across multiple data sources, which are also subject to privacy constraints, is studied under federated learning \citep{konevcnyEtal-16arXiv, McmahanEtal-16arXiv}. Therefore, we refer to distributed communication-efficient GAN as federated GAN (FedGAN). FedGAN can extend GAN applications to federated learning. For example, in lots of cases, even the pooled dataset is not large enough to learn an accurate model, and FedGAN can help producing more data similar to the original data for better training \citep{bowles2018gan}.

The major challenge in GAN algorithms is their convergence since cost functions may not converge using gradient descent in the minimax game between the discriminator and the generator. Convergence is also the major challenge in federated learning since each source updates local model using multiple stochastic gradient descents (SGDs) before syncing with others through the intermediary \citep{dinh2019federated}; it becomes even more challenging when data at different sources are not independent and identically distributed (non-iid) \citep{wang2019adaptive}. Convergence of distributed GANs is an open problem. We theoretically prove our FedGAN algorithm converges even with non-iid  sources under equal and two time step updates. We connect results from stochastic approximation for GAN convergence and communication-efficient SGD for federated learning to address FedGAN convergence.

We experiment our FedGAN on popular toy examples in the GAN literature including 2D system, mixed Gaussians, and Swiss role, on image datasets including MNIST, CIFAR-10, and CelebA, and on time series data form energy industry including household electricity consumption and electric vehicles charging, to show its convergence, efficiency, and robustness to reduced communications. We use energy industry since it is not currently studied in the federated learning but often involves distributed data, subject to privacy constraints, with limited communication infrastructure, and mostly in time series format \citep{stankovic2016measuring,balachandran2014bandwidth}. To the best of our knowledge, there is no other algorithm for communication efficient distributed GAN to compare ours with. We compare the performance of FedGAN with a typical distributed GAN that has local discriminators and one generator in the intermediary which communicate frequently (similar to that in \citep{augenstein2019generative}).

The rest of the paper is as follows. We first review the relevant literature (Section \ref{sec:liter_rev}). Then we propose our FedGAN algorithm (Section \ref{sec:algo}), discuss its communication and computation properties (Section \ref{sec:communi_comput}), and theoretically prove its convergence even when the data sources are non-iid (Section \ref{sec:conv_analy}). Next, we run experiments on toy examples (Section \ref{sec:toy}), image dataset (Section \ref{sec:exp_img}), and time series dataset (Section \ref{sec:exp_energy}). Finally, we point out to some observations in our experiments and some open problems (Section \ref{sec:concl}). Part of proofs and experiments are in appendices.
	
\section{Literature Review}
\label{sec:liter_rev}
This work relates to three literature, GAN convergence, distributed GAN, and federated learning.

Convergence of GAN has been studied through convergence of dynamical systems using stochastic approximation by modeling GAN as a continuous two-player game solvable by gradient-based multi-agent learning algorithms \citep{chasnov2019convergence}. With equal step sizes of generator and discriminator updates the problem is single time-scale stochastic approximation \citep{konda2004convergence} for which  certain conditions are developed (and tested \citep{mescheder2018training}) for convergence of the ODE representing the stochastic approximation of GAN, e.g. Hurwitz Jacobian at equilibrium \citep{khalil2002nonlinear}, negative definite Hessians with small learning rate
\citep{nowozin2016f,ratliff2013characterization}, consensus optimization regularization  \citep{nagarajan2017gradient}, and non-imaginary eigenvalues of the spectrum of the gradient vector field Jacobian \citep{mescheder2017numerics}. With different step sizes of generator and discriminator updates the problem is two time-scale stochastic approximation  \citep{borkar1997stochastic} for which convergence is shown under  global \citep{borkar1997stochastic} and local asymptotic stability assumptions  \citep{karmakar2017two,borkar2009stochastic}. 
\citep{heusel2017gans} proposes a two time-scale update rule (TTUR) for GAN with SGD and shows convergence under the those stability conditions. All of the above papers are for centrally trained GAN, while our FedGAN algorithm is distributed.

Distributed GANs are proposed recently. For iid data sources, \citep{hardy2018md} proposes a single generator at the intermediary and distributed discriminators which communicate generated data and the corresponding error. Also, discriminators exchange their parameters occasionally to avoid overfitting to local data. \citep{hardy2018gossiping} utilizes a gossip approach for distributed GAN which does not require an intermediary server. For non-iid data sources, \citep{yonetani2019decentralized} trains individual discriminators and updates the centralized generator to fool the weakest discriminator. All of the above algorithms require large communications, while our FedGAN is communication efficient. Also, to the best of our knowledge there is no theoretical result for convergence of distributed GAN, while we provide such results for FedGAN.

Federated learning, proposed for communication efficient distributed learning, runs parallel SGD on randomly selected subset of all the agents and updates the parameters with the averages of the trained models through an intermediary once in a while. Its convergence is proved for convex objective with iid data \citep{stich2018local},  non-convex objective with iid data
\citep{wang2018cooperative}, strongly convex objective and non-iid data with all responsive agents
 \citep{wang2019adaptive} and some non-responsive agents  
 \citep{li2019convergence}, and non-convex objective and non-iid data
\citep{yu2019parallel}. Our FedGAN study is with non-convex objective and non-iid data, and also involves GAN convergence challenges on top of distributed training issues.

A distributed GAN with communication constraints is proposed by \cite{augenstein2019generative} under FedAvg-GAN which has distributed discriminators but centralized generator, similar to distributed GAN in \cite{hardy2018md} with the difference of selecting a subset of agents for discriminator updating. This approach does not fully address the large communications required for distributed GAN as it needs communications in each generator update iteration. We overcome this issue by placing both the discriminators and the generators at the agents, and then communicating only every $K$ steps with intermediary to sync parameters across agents. An architecture similar to our FedGAN is envisioned in  \citep{Rajagopal2019FederatedAL} and \citep{hardy2018md}, but they do not provide theoretical studies for convergence, and their experiment results are very limited. We provide  a complete study in this paper.

\section{FedGAN Algorithm}
In this section we propose FedGAN algorithm, discuss its communication and computation complexity, and theoretically prove its convergence.

\subsection{Model and Algorithm}
\label{sec:algo}
We denote the training iteration horizon by $N$ and index time by $n$. Consider agents $\{1,2,..., B\}$  with local dataset of agent $i$ denoted by $\mathcal{R}_i$ and weight of  agent $i$ denoted by $p_i:=\frac{|\mathcal{R}_i|}{\sum_{j=1,...,n} |\mathcal{R}_j|}$.  $\mathcal{R}_i$ data comes from an individual distributions for agent $i$ (data in non-iid across agents). Assume each agent has local discriminator and generator with corresponding parameter vectors $\vw_n^i$ and $\vtheta_n^i$, loss functions $\mathcal{L}^i_D$ and $\mathcal{L}^i_G$, local true gradients $\vh^i(\vtheta^i_n, \vw^i_n)$  and $\vg^i(\vtheta^i_n, \vw^i_n)$, local stochastic gradients $\tilde{\vg}^i(\vtheta^i_n, \vw^i_n)$ and $\tilde{\vh}^i(\vtheta^i_n, \vw^i_n)$, and learning rates $a(n)$ and $b(n)$ at time $n$. We assume the learning rates are the same across agents. The gradients $\tilde{\vh}^i(\vtheta, \vw)$ and $\tilde{\vg}^i(\vtheta, \vw)$ are stochastic, since every agent uses a mini-batch of his local data for SGD.

There is an intermediary whose role is syncing the local generators and discriminators. The intermediary parameters at time $n$ are denoted by $\vw_n$ and $\vtheta_n$. Note that the intermediary does not train a generator or discriminator itself, and $\vw_n$ and $\vtheta_n$ are only obtained by averaging $\vtheta^i_n$ and $\vw^i_n$ across $i$.


The FedGAN algorithm is presented in Algorithm \ref{alg: fedgan}. All agents run SGDs for training local generators and discriminators using local data. Every $K$ time steps of local gradient updates, the agents send their parameters to the intermediary which in turn sends back the average parameters to all agents to sync. We refer to $K$ by synchronization interval. $a(n)$, $b(n)$ and  $K$ are tuning parameters of the FedGAN algorithm.

\begin{remark} In our model privacy is the main reason agents do not share data and rather send model parameters. Adding privacy noise to the model parameters can further preserve privacy. We leave this as a future direction for this research. Also, we assume all agents participate in the communication process. There is a literature on federated learning which studies if only part of the agents send their parameters due to communication failures \citep{konevcnyEtal-16arXiv}. This could be an extension to this paper for FedGAN.
\end{remark}

	\begin{algorithm}[tbp]
		\SetAlgoLined
		\SetKwInOut{Input}{Input~}
		\Input{Set training period $N$. Initialize local discriminator and generator parameters for each agent $i$: $\vw^i_{0}=\hat{\vw}$ and $\vtheta^i_{0}=\hat{\vtheta}$, $\forall i \in \{1,2,...,B\}$. Set the learning rates of the discriminator and the generator at iteration $n$, $a(n)$ and $b(n)$ for $n=1,2,...,N-1$.  Also set synchronization interval $K$.}
		\For{$n=1,2,...,N-1$}{
			Each agent $i$ calculates local stochastic gradient $\tilde{\vg}^i({\vtheta}_{n}^i, {\vw}_{n}^i)$  from $\mathcal{R}_i$ and $\tilde{\vh}^i({\vtheta}_{n}^i,{\vw}_{n}^i)$ from $\mathcal{R}_i$ and fake data generated by the local generator.
			
			Each agent $i$ updates its local parameter in parallel to others via
			\begin{equation}
			\label{eq:algo_update}
			\left\{ \begin{array}{l}
			\vw_{n}^i = \vw_{n-1}^i+a(n-1) \tilde{\vg}_i({\vtheta}_{n-1}^i, {\vw}_{n-1}^i)\\
			\vtheta_{n}^i = \vtheta_{n-1}^i+b(n-1) \tilde{\vh}_i({\vtheta}_{n-1}^i, {\vw}_{n-1}^i)
			\end{array} \right.
			\end{equation}\
			
			\If{$n\mod K=0$ }{
				All agents send parameters to intermediary.
				
				The intermediary calculates $\vw_n$ and $\vtheta_n$ by averaging			\begin{equation}
				\label{eq:algo_2}
				\vw_n\overset{\Delta}{=}\sum_{j=1}^{B} p_j  {\vw}_{n}^j, \quad
				 \vtheta_n\overset{\Delta}{=}\sum_{j=1}^{B} p_j {\vtheta}_{n}^j
				\end{equation}
				The intermediary send back $\vw_n, \vtheta_n$ and agents update local parameters
				\begin{equation}
				\vw_n^i =\vw_n, \quad
				\vtheta_n^i = \vtheta_n
				\end{equation}
			}
		}
		\caption{Federated Generative Adversarial Network (FedGAN)}
		\label{alg: fedgan}
	\end{algorithm}

	\subsection{Communication and Computation Complexity}
	\label{sec:communi_comput}
FedGAN communications are limited to sending parameters to intermediary by all agents and receiving back the synchronized parameters every $K$ steps. For a parameter vector of size $M$ (we assume the size of generator and discriminator are the same order), the average communication per round per agent is $\frac{2\times 2M}{K}$. Increasing $K$ reduces the average communication, which may reduce the performance of trained FedGAN (we experiment FedGAN robustness to increasing $K$ in Section \ref{sec:exp_img} and leave its theoretical understanding for future research). For a general distributed GAN where the generator is trained at the intermediary, the communication involves sending discriminator parameters and generator parameters (or the fake generated data), and this communication should happen at every time step for convergence. The average communication per round per agent therefore is $2\times 2M$. This shows the communication efficiency of FedGAN.

Since each agent trains a local generator, FedGAN requires increased computations for agents compared to distributed GAN, but at the same order (roughly doubled). However, in FedGAN the intermediary has significantly lower computational burden since it only average the agents' parameters.

	\subsection{Convergence Analysis}
	\label{sec:conv_analy}
	
	In this section, we show that FedGAN converges even with non-iid sources of data, under certain standard assumptions. We analyze the convergence of  FedGAN for both equal time-scale updates and two time-scale updates (distinguished by whether $a(n)=b(n)$). While using equal time-scale update is considered standard, some recent progress in GANs such as Self-Attention GAN \citep{zhang2018self}  advocate the use of two time-scale updates presented in \citep{heusel2017gans} for tuning hyper-parameters.
	
	We extend the notations in this section. For a centralized GAN that pools all the distributed data together, we denote the generator's and discriminator's loss functions by $\mathcal{L}_G$ and $\mathcal{L}_D$,  with true gradients $\vh(\vtheta, \vw):=\nabla _\vtheta \mathcal{L}_G$ and $\vg(\vtheta, \vw)=\nabla _\vw \mathcal{L}_D$. Also define ${\mM}^{(\vtheta)}:=\vh(\vtheta, \vw)-\sum_{i} p_i\tilde{\vh}^i(\vtheta, \vw)$ and ${\mM}^{(\vw)}:= \vg(\vtheta, \vw)-\sum_{i} p_i\tilde{\vg}^i(\vtheta, \vw)$.
	${\mM}^{(\vtheta)}$ and ${\mM}^{(\vw)}$
	are random variables due randomness in mini-batch stochastic gradient of $\tilde{\vh}^i,\tilde{\vg}^i$.
	
	We make the following standard assumptions for rest of this section. The first four are with respect to the centralized GAN and are often used in stochastic approximation literature of GAN convergence. The last assumption is with respect to local GANs and is common in distributed learning.
	 
	\begin{itemize}
		\item [(\textbf{A1})] $\vh^i$  and $\vg^i$ are $L$-Lipschitz.
		\item [(\textbf{A2})] $\sum_n a(n)=\infty$, $\sum_n a^2(n)<\infty$, $\sum_n b(n)=\infty$, $\sum_n b^2(n)<\infty$ 
		\item [(\textbf{A3})] The stochastic gradient errors $\{{\mM}^{(\vtheta)}_{n}\} $ and $\{ {\mM}_{n}^{(\vw)}\}$ are martingale difference sequences w.r.t. the increasing $\sigma$-filed $\mathcal{F}_n=\sigma(\vtheta_l, \vw_l, {\mM}^{(\vtheta)}_{l}, {\mM}_{l}^{(\vw)}, l\leq n),n\geq 0$.  
		\item [(\textbf{A4})] $\sup_n||{\vtheta_n}|| < \infty$ and $\sup_n||{\vw_n}|| < \infty$.
		\item [(\textbf{A5})] $\E||[\tilde{\vg}^i(\vtheta, \vw)]-\vg^i(\vtheta, \vw)||\leq \sigma_g$, $\E||[\tilde{\vh}^i(\vtheta, \vw)]-\vh^i(\vtheta, \vw)||\leq \sigma_h$ (bounded variance) and  $||\vg^i(\vtheta, \vw)-\vg(\vtheta, \vw)||\leq \mu_{g}$ (bounded gradient divergence).
	\end{itemize}
	\textbf{(A1)}-\textbf{(A4)} are clear assumptions. In \textbf{(A5)}, the first bound ensure the closeness between the local stochastic gradients and local true gradients, while the second bound ensures closeness of local discriminator true gradient of non-iid sources and the discriminator true gradient of the pooled data.
	
	We next prove the convergence of FedGAN. To this end, we rely on the  extensive literature that connects the convergence of GAN to the convergence of an ODE representation of the parameter updates \citep{mescheder2017numerics}. We prove the ODE representing the parameter updates of FedGAN asymptotically tracks the ODE representing the parameter update of the centralized GAN. We then use the existing results on convergence of centralized GAN ODE \citep{mescheder2018training, nagarajan2017gradient}. Note that our results do not mean the FedGAN and centralized GAN converge to the same point in general. 

	\textbf{Equal time-scale update}. It has been shown in \citep{nagarajan2017gradient, mescheder2017numerics} that under equal time-scale update, the centralized GAN tracks the following ODE asymptotically (we use $t$ to denote continuous time). 
	\begin{equation}
	\label{eq: ode_z}
	\begin{pmatrix}\dot{{\vw}}(t)\\ \dot{{\vtheta}}(t)\end{pmatrix} = 
	\begin{pmatrix}{\vg}({\vtheta}(t),{\vw}(t))\\ {\vh}({\vtheta}(t),{\vw}(t))\end{pmatrix}.
	\end{equation}
	
    We now show that when $a(t)=b(t)$, the updates  in (\ref{eq:algo_update}) as specified by the proposed algorithm, also tracks the ODE in (\ref{eq: ode_z}) asymptotically. To this end, we further extend the notations. For centralized GAN, Let $\vz(t):=(\vw(t),\vtheta(t))^\top$ and define $\vq(\vz(t)):=\dot{\vz}(t)=(\vg(\vtheta(t)), \vh(\vw(t)))^\top$ from (\ref{eq: ode_z}). For FedGAN, define $\vz_n:=(\vw_n,\vtheta_n)^\top$ from (\ref{eq:algo_2}) and with a little abuse of notation, define time instants $t(0)=0$, $t(n)=\sum_{m=0}^{n-1} a(m)$. Define a continuous piece-wise linear function
	$\bar{\vz}(t)$ by 
	\begin{equation}\label{eq: z bar}
	\bar{\vz}(t(n)):={\vz}_{n},
	\end{equation}
	with linear interpolation on each interval $[t(n), t(n+1)]$. Correspondingly, define $\bar{\vw}(t(n))$ and $\bar{\vtheta}(t(n))$ to have  $\bar{\vz}(t(n))=(\bar{\vw}(t(n)),\bar{\vtheta}(t(n)))^\top$.

	Let $\vz^s(t)$ (correspondingly $\vz_s(t)$) denote the unique solution to (\ref{eq: ode_z}) starting (ending) at $s$
	\begin{equation}
	\label{eq:ode_z_s}
	\dot{{\vz^s}}(t) = {\vq}({\vz^s}(t)), t\geq s \quad (t\leq s)
	\end{equation}
	with $\vz^s(s)=\bar{\vz}(s)$ (with $\vz_s(s)=\bar{\vz}(s)$).	Define $(\vw^s(t), \vtheta^s(t)$ (correspondingly  $(\vw_s(t), \vtheta_s(t)$)  to be the elements of  $\vz^s(t)$ (elements of $\vz_s(t)$).
	
	Now, in order to prove FedGAN follows (\ref{eq: ode_z}) asymptotically, it is sufficient to show that $\bar{\vz}(t)$ asymptotically tracks $\vz^s(t)$ and $\vz_s(t)$ as $s\to \infty$. The first step is to show the difference between the intermediary averaged parameters ${\vw}_{n},{\vtheta}_{n}$ and $\vv_n,\vphi_n$ defined below base on the centralized GAN updates in between synchronization intervals is bounded. If $n=\ell K$, then let $\vv_{n}=\vw_{n}$ and $\vphi_{n}=\vtheta_{n}$ otherwise, denote $n_1$ to be the largest multiplication of $K$
 before $n$ and let
    \begin{align}
	  \vv_n = \vw_{n_1} + \sum_{k=n_1}^n a(k) \vg(\vphi_k, \vv_k), \quad
	  \vphi_n = \vtheta_{n_1} + \sum_{k=n_1}^n b(k) \vh(\vphi_k, \vv_k).
	\end{align}
	We prove the following \lemref{lem: bound_var} and \lemref{lem: bound_grad_div} for this purpose. The distinction between proof here and in Theorem 1 in \citep{wang2019adaptive} is that we consider local SGD for federated learning. 
	
	We present the proofs of result in this Section in Appendix \ref{app: proof}.
	
	\begin{lemma}
		\label{lem: bound_var}
		$\E ||\vw^i_n-\vv_n||+\E||\vtheta^i_n-\vphi_n||\leq r_1(n):= \frac{\sigma_g+\mu_g+\sigma_h}{2L}[(1+2a(n-1)L)^{n\, \text{mod}\, K}-1]$.  
	\end{lemma}
	
	\begin{lemma}
		\label{lem: bound_grad_div}
		$\E||{\vw}_{n}-\vv_{n}||+\E||{\vtheta}_{n}-\vphi_{n}||\leq r_2(n):=\frac{(\sigma_g+\sigma_h+\mu_g)}{2L}[(1+2a(n-1)L)^{K}-1]-a(n-1)\mu_gK $.
	\end{lemma}
	
	Next, using  \lemref{lem: bound_var} and \lemref{lem: bound_grad_div}, we prove \thmref{thm: equal_traj} which shows $\bar{\vz}(t)$ asymptotically tracks $\vz^s(t)$ and $\vz_s(t)$ as $s\to \infty$. This in turn proves (\ref{eq:algo_update}) asymptotically tracks the  limiting ODE in (\ref{eq: ode_z}). The proof is modified from Lemma 1 in Chapter 2 of \citep{borkar2009stochastic} from centralized GAN to FedGAN.
	\begin{theorem}
		\label{thm: equal_traj}
		For any $T>0$ (a.s. stands for almost surly convergence)
		\begin{equation} 
		\lim_{s\rightarrow \infty} \sup_{t\in [s,s+T]} ||\bar{\vz}(t)-\vz^s(t)|| = 0, a.s.,\quad
		\lim_{s\rightarrow \infty} \sup_{t\in [s-T,s]} ||\bar{\vz}(t)-\vz_s(t)|| = 0, a.s. 
		\end{equation}
	\end{theorem}
	\begin{corollary}
	From Theorem \ref{thm: equal_traj} above and Theorem 2 of Section 2.1 in \cite{borkar2009stochastic}, under equal time-scale update, FedGAN tracks the ODE in (\ref{eq: ode_z}) asymptotically.
	\end{corollary}
	
We provide the convergence analysis of FedGAN with two time-scale updates in Appendix \ref{app: two_time}.

\section{Experiments}
In this section we experiment the proposed FedGAN algorithm using different datasets to show its convergence, performance in generating close to real data, and robustness to reducing communications (by increasing synchronization interval $K$). First in Section \ref{sec:toy} we experiment with popular toy examples in the GAN literature, 2D system \citep{nagarajan2017gradient}, mixed Gaussian \citep{Metz2016UnrolledGA} and Swiss Roll \citep{gulrajani2017improved}. Next, in Section \ref{sec:exp_img} we experiment with image datasets including MNIST, CFAR-10 and CelebA. Finally we consider time-series data in Section \ref{sec:exp_energy} by focusing on energy industry, including PG\&E household electricity consumption and electric vehicles charging sessions from a charging station company. In all the experiments, data sources are partitioned into non-iid subsets each owned by one agent.

	\subsection{Toy Examples}
	\label{sec:toy}
	
	Three toy examples, 2D system mixed Gaussian and Swiss role, are presented in  Appendix \ref{app: toy}. These experiments show the convergence and performance of the FedGAN in generating data similar to real data. The first experiment, 2D system, also shows the robustness of FedGAN performance to increasing synchronization intervals $K$ for reducing communications.

	\subsection{Image Datasets}
	\label{sec:exp_img}
	We test  FedGAN on MNIST, CIFAR-10, and CelebA to show its performance on image datasets. 
	
	Both MNIST and CIFAR-10 consist of $10$ classes of data which we split across $B=5$ agents, each with two classes of  images. We use the ACGAN neural network structure in \citep{odena2017conditional}. For a detailed list of architecture and hyperparameters and other generated images see  Appendix \ref{app: net_hyper}.
	
For the MNIST dataset, we set synchronization interval $K=20$. Figure \ref{fig:mnist} presents the generated images of MNIST from FedGAN. It shows FedGAN can generate close to real images.

\begin{figure}[htbp]
		\begin{subfigure}[b]{0.5\textwidth}
			\centerline{\includegraphics[scale=0.4]{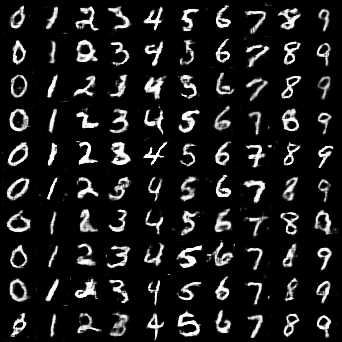}}
			\caption{Generated MNIST images from FedGAN with\\ $B=5$ agents and synchronization interval $K=20$.}
		    \label{fig:mnist}
		\end{subfigure}
		\begin{subfigure}[b]{0.5\textwidth}
			\centerline{\includegraphics[scale=0.5]{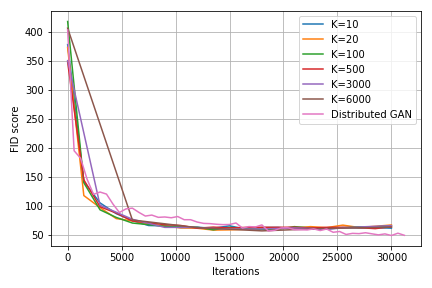}}
			\caption{FID score  on CIFAR-10 with $B=5$ and $K=10, 20, 100, 500, 3000, 6000$ and distributed GAN.}
			\label{cifar10_fid}
		\end{subfigure}
		\caption{Experiment results for FedGAN on MNIST and CIFAR-10.}
	\end{figure}
	
	For the CIFAR-10 dataset, we use the FID scores \citep{karmakar2017two} to compare the generated and real data and show the realness of the generated images from FedGAN. We check FedGAN performance robustness to reduced communications and increased synchronization intervals $K$ by setting $K=10, 20, 1000, 500, 3000, 6000$. We benchmark FedGAN performance against a typical distributed GAN  similar to that in \citep{augenstein2019generative}, where there are local discriminators that send the parameters to the intermediary after each update, and the intermediary sends back the average discriminator parameters plus the generated data of its updated centralized generator (to the best of our knowledge, there is no other algorithm for communication efficient distributed GAN to compare ours). 	Figure  \ref{cifar10_fid} shows the results for FedGAN CIFAR-10. It can be observed that, even for large synchronization interval $K$, the FedGAN FID score is close to distributed GAN (except for the tail part). This indicates that FedGAN has high performance for image data, and furthermore its performance is robust to reducing the communications by increasing synchronization intervals $K$. The gap between the distributed GAN and FedGAN in the tail part required further investigation in the future. 
	
    Next, we experiment FedGAN algorithm on CelebA \citep{liu2015deep}, a  dataset of $202,599$ face images of celebrities. We split data across $B=5$ agents by first generating $16$ classes based on the combinations of four binary attributes of the images, Eyeglasses, Male, Smiling, and Young, and then allocating each of these $16$ classes to one of the $5$ agents (some classes divided between two agents to ensure equal size of data across agents). We check FedGAN robustness to reduced communications by setting synchronization intervals $K=10, 20, 50, 100, 200$, and also compare the performance with distributed GAN. We use the ACGAN neural network structure in \citep{odena2017conditional} (for details of data splitting, lists of structures and hyperparameters see  Appendix \ref{app: net_hyper}). Figure \ref{fig:celebA} shows the generated images with $K=50$ with $N=16000$ iterations (more generated images are presented in Appendix \ref{app: net_hyper}). Figure \ref{fig:celebA_fid} shows the performance of FedGAN is close to distributed GAN, and robust to reduced communications.
	
	\begin{figure}[htbp]
		\begin{subfigure}[b]{0.45\textwidth}
			\centerline{\includegraphics[scale=0.4]{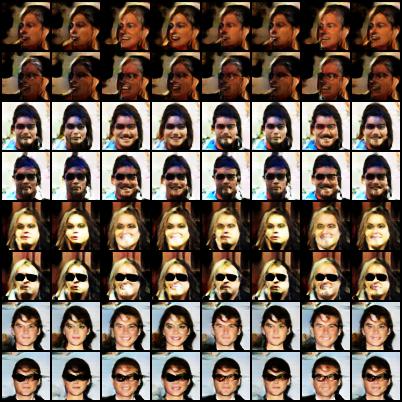}}
			\caption{Generated images with $B=5$, $K=50$, and $N=16000$ iterations.}
		    \label{fig:celebA}
		\end{subfigure}
		\begin{subfigure}[b]{0.45\textwidth}
			\centerline{\includegraphics[scale=0.45]{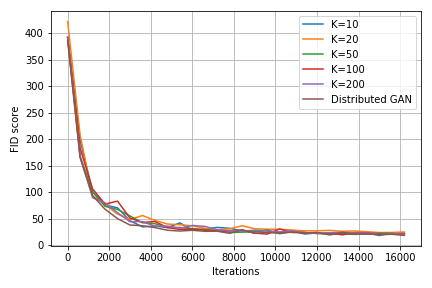}}
			\caption{FID score  on CelebA, $B=5$ and $K=10, 20, 50, 100, 200$ compared to distributed GAN.}
			\label{fig:celebA_fid}
		\end{subfigure}
		\caption{Experiment results for FedGAN on CelebA.}
	\end{figure}


	\subsection{Time Series Data in Energy Industry}
	\label{sec:exp_energy}
    In this section, we experiment FedGAN for time series data. We particularly focus on energy industry where the communication and privacy constraints are important, and data often is in time series format. We experiment on household electricity consumption and electric vehicle (EV) charing sessions data. In order to measure the performance of FedGAN for time series data, absent of an accepted measure or score for measuring the realness of time series data, we cluster both the real data and generated data, and visually compare the top 9 cluster centroids for each.

	For household electricity consumption, we use the hourly load profile data of $500k$ households in one year in California by Pacific Gas and Electric Company (PG\&E) (dividing every single households data into separate daily profiles). The data includes both household characteristics and temporal features including regional average income, enrolled in low income tariff program/not, all electric/not, daily average temperature, weekday/weekend, month, tariff type, climate zone, house premise type. 
	
	For EV data, we use data from an electric vehicle (EV) charging stations company including $12.4$ million charging sessions where each session is defined by the plug-in and -off of an EV at a charging station. For each session, we observe  start time,  end time, 15-min charging power, charging time, and charges energy, and fully charged or not. We also observe characteristics of the charging station as well as the EV. For example an EV with battery capacity 24kW arriving at a high-tech workplace at 9:00am on Monday (see Appedix \ref{app: data} for full detail).
	
	For both experiments, we split the data in equal parts across $B=5$ agents (representing different utility companies or different EV charging companies), based on climate zones or category of charging stations (to ensure non-iid data across agents), and set synchronization interval $K=20$. We use a network structure similar to CGAN \citep{mirza2014conditional}.

    We separate $10\%$ of each agent's data, train FedGAN on the rest of the data, and use the trained FedGAN to generate fake time series profiles for those $10\%$. We then apply k-means for both the real and generated data of those $10\%$. The k-means top 9 centroids are shown in Figure \ref{cluster_real} and  \ref{cluster_test} for household electricity consumption, and in Figure \ref{cluster_ev} for EV charging sessions. Visually comparing them shows the performance of FedGAN on generating close to real profiles for time series data.

	\begin{figure}[ht]
		\begin{subfigure}[b]{0.5\textwidth}
			\centerline{\includegraphics[scale=0.45]{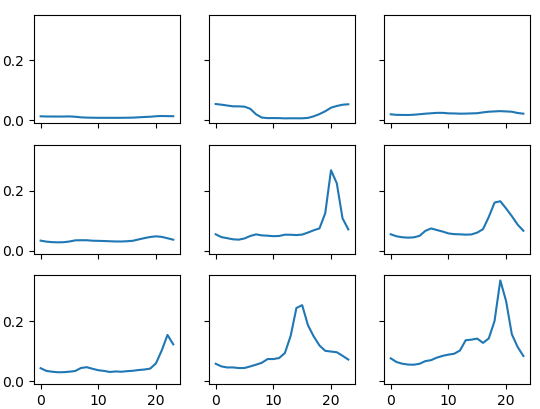}}
			\caption{Top 9 cluster centroids of the real data}
			\label{cluster_real}
		\end{subfigure}
		\begin{subfigure}[b]{0.5\textwidth}
			\centerline{\includegraphics[scale=0.45]{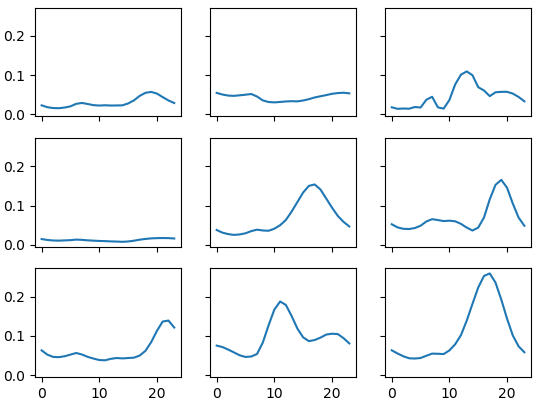}}
			\caption{Top 9 cluster centroids of the generated data}
	\label{cluster_test}
		\end{subfigure}
		\caption{Top 9 k-means clusters for real PG\&E daily household electricity consumption, and FedGAN generated profiles with $B=5$  and $K=20$. The consumption profiles are normalized.}
		\label{cluster}
	\end{figure}

	\begin{figure}[htbp]
		\begin{subfigure}[b]{0.5\textwidth}
			\centerline{\includegraphics[scale=0.45]{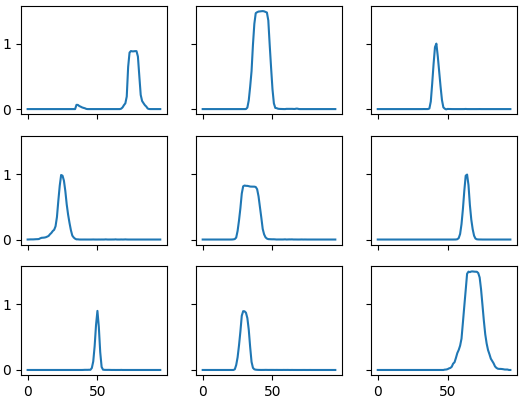}}
			\caption{Top 9 cluster centroids of the real data}
			\label{cluster_ev_real}
		\end{subfigure}
		\begin{subfigure}[b]{0.5\textwidth}
			\centerline{\includegraphics[scale=0.45]{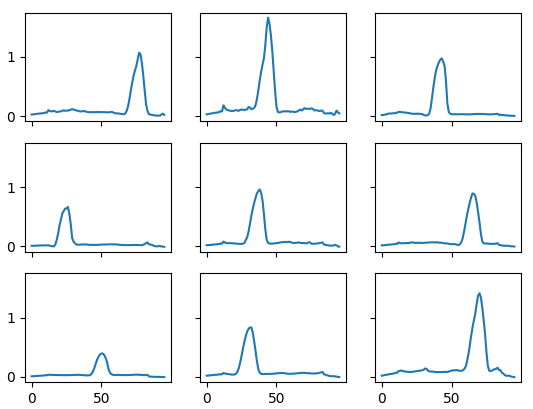}}
			\caption{Top 9 cluster centroids of the generated data}
			\label{cluster_ev_test}
		\end{subfigure}
		\caption{Top 9 k-means clusters for real EV charging profiles, and FedGAN generated profiles with $B=5$  and $K=20$. The charging profiles are normalized.}
		\label{cluster_ev}
	\end{figure}

\section{Conclusions and Future Directions}
We proposed an algorithm for communication-efficient distributed GAN subject to privacy constraints (FedGAN). We proved the convergence of our FedGAN for non-iid sources of data both with equal and two time-scale updates of generators and discriminators. We experimented FedGAN on toy examples (2D,  mixed Gaussian, and Swiss roll), image data (MNIST, CIFAR-10, and CelebA), and time series data (household electricity consumption and EV charging sessions) to study its convergence and performance. We showed FedGAN has similar performance to general distributed GAN, while reduces communication complexity. We also showed its robustness to reduced communications.

There are some observations and open problems. First is experimenting FedGAN with other federated learning datasets such as mobile phone texts, and with other applications besides image classification and energy. Robustness to increasing agents number $N$ is required which requires large set of GPUs (beyond the engineering capacity of this research). Theoretically, an explanation for the FedGAN robustness to reduced communication, as well as identifying the rate of convergence are interesting. While privacy is the main reason agents do not share data in our FedGAN, adding privacy noise to the model parameters, for example by differential privacy, can further preserve privacy and should be studied. Finally, non-responsiveness of some agents in practice should be studied.

\label{sec:concl}

\section{Broader Impact}
This work has the following potential positive impact in the society: it provides an algorithm for shared learning across agents with local data while preserving privacy and low communication cost, hence it helps democratizing data power. We have also emphasized energy domain as an application which is at the forefront of sustainability and reversing global warming; we particularly experimented on data for household demand prediction and electric vehicle charging station planning.

\section*{Acknowledgment}

We would like to thank the Pacific Gas and Electric Company (PG\&E) for providing the household energy consumption dataset and SLAC National Accelerator Laboratory for providing the EV dataset.


\bibliography{example_paper}
\bibliographystyle{plainnat}

\onecolumn
\appendix

\section*{Supplementary Material for "FedGAN: Federated Generative Adversarial Networks for Distributed Data"}

\section {Convergence Analysis of FedGAN with Two Time-Scale Updates} \label{app: two_time}
To study the two time-scale FedGAN we add the following assumptions.
	\begin{itemize}
		\item [(\textbf{A6})]  $b(n)=\text{o}(a(n))$
		\item [(\textbf{A7})]  For each ${\vtheta}$, the ODE $\dot{{\vw}}(t)={\vg}({\vtheta}, {\vw}(t))$ has a  local asymptotically stable attractor ${\lambda}({\vtheta})$ within a domain of attraction $H_\vtheta$ such that ${\vtheta}$ is Lipschitz. The ODE $\dot{{\vtheta}}(t)=\vh({\vtheta}(t),{\lambda}({\vtheta}(t)))$ has a local asymptotically stable equilibrium ${\vtheta^*}$ within a domain of attraction.
	\end{itemize}
	\textbf{(A6)} is the standard assumption to determine the relationship of the generator and discriminator learning rates. \textbf{(A7)} characterizes the local asymptotic behavior of the limiting ODE in (\ref{eq: ode_z}) and shows its local asymptotic stability.
	Both assumptions are regular conditions in the literature of two time-scale stochastic approximation \citep{borkar2009stochastic,karmakar2017two}. 
    In the literature of stochastic approximation, often global asymptotic stability assumptions are made but \citep{karmakar2017two} and Chapter 2 of \citep{borkar2009stochastic} relax them to local asymptotic stability which is a more practical assumption in GANs\footnote{Note that another way of relaxation to local asymptotic stability is to assume that the initial parameter is in the region of attraction for a locally asymptotically stable attractor, which is difficult to ensure in practice.}. The relaxed local stability assumption \textbf{(A7)} limits the convergence results to be conditioned on an unverifiable event i.e. $\{\vw_n\}$ and $\{\vtheta_n\}$ eventually belongs to some compact set of their region of attraction.

	To prove the convergence of FedGAN for two time-scale updates, similar to the proof for equal time-scale update, we show the ODE representation of the FedGAN asymptotically tracks the ODE below representing the parameter update of the two time-scale centralized GAN
	\begin{equation}
	\label{eq:ode_2_w}
	\dot{{\vw}}(t) = \frac{1}{\epsilon}{\vg}(\vtheta(t), {\vw}(t)),\quad
	\dot{{\vtheta}}(t) = {\vh}(\vtheta(t), {\vw}(t))
	\end{equation}
    where $\epsilon \downarrow 0$ to ensure updating ${\vw}(t)$ is fast compared to updating $\vtheta(t)$  \textbf{(A6)}. Consequent to \textbf{(A6)},  $\vtheta(t)$ can be considered quasi-static while analyzing the updates of ${\vw}(t)$ and we can look at the following ODE in studying ${\vw}(t)$ (with small change of notation we drop the notational dependency on fixed $\vtheta$ for convenience) 
	\begin{equation}
	\dot{{\vw}}(t) = {\vg}({\vw}(t)).
	\label{eq: sample_ode}
	\end{equation}
	In order to show that the updates of ${\vw}_n$ asymptotically tracks (\ref{eq: sample_ode}), we follow the same idea as in equal time-scale update to construct the continuous interpolated trajectory $\bar{\vw}(t)$ defined immediately after (\ref{eq: z bar}), and show that it asymptotically almost surely approaches the solution of (\ref{eq: sample_ode}). For this, we also use the construction of $\vw^s(t)$ and $w_s(t), t\leq s$ defined immediately after (\ref{eq:ode_z_s}). We thus have the following lemma as a special case of \thmref{thm: equal_traj}.
	\begin{lemma}
		\label{lem: asym_traj}
		For any $T>0$,
		\begin{equation} 
		\lim_{s\rightarrow \infty} \sup_{t\in [s,s+T]} ||\bar{\vw}(t)-\vw^s(t)|| = 0, a.s.,\quad
		\lim_{s\rightarrow \infty} \sup_{t\in [s-T,s]} ||\bar{\vw}(t)-\vw_s(t)|| = 0, a.s.
		\end{equation}
	\end{lemma}
	
	\begin{proof}[Proof of \lemref{lem: asym_traj}]
		The proof can be directly obtained by consider the special case of \thmref{thm: equal_traj} where the dimension of $\vtheta$ is zero.
	\end{proof}

	With  \lemref{lem: asym_traj}, \thmref{thm: single_converge} below shows the FedGAN tracks the ODE in (\ref{eq: sample_ode}) asymptotically when ${\vtheta}_{n}$ is fixed. We refer the reader for proof to \citep{borkar2009stochastic}.  
	\begin{theorem}
		\label{thm: single_converge}
		[Theorem 2, Chapter 2, \citep{borkar2009stochastic}]. Almost surely, the sequence $\{\vw_n\}$ generated by  (\ref{eq:algo_update}) when ${\vtheta}_{n}$ is fixed to ${\vtheta}$ converges to a (possibly sample path dependent) compact connected internally chain transitive invariant set of (\ref{eq: sample_ode}). 
	\end{theorem}
The following Lemma \ref{lem: trans_invar} and Theorem \ref{thm: two_time} extend the results of Theorem \ref{thm: single_converge} to the case when both ${\vw}_{n}$ and ${\vtheta}_{n}$ could vary as in (\ref{eq:ode_2_w}). Both proofs should be referred to the respective part of \citep{borkar2009stochastic}.  Lemma \ref{lem: trans_invar} shows that $\vw_n$ asymptotically tracks $\lambda(\vtheta_n)$, and Theorem \ref{thm: two_time} shows the convergence of the proposed FedGAN asymptotically converges to $(\lambda(\vtheta^*), \vtheta^*)$, which is an equilibrium of the ODE in (\ref{eq:ode_2_w})  representing centralized GAN with two time-scale updates. Lemma \ref{lem: trans_invar} is an adaption from Lemma 1 in Chapter 6 of \citep{borkar2009stochastic}.
 	\begin{lemma}
	\label{lem: trans_invar}
		 For the two time-scale updates as specified in (\ref{eq:algo_update}), $(\vw_n,\vtheta_n)\rightarrow \{(\lambda(\vtheta), \vtheta)\}$ almost surely where $\{(\lambda(\vtheta), \vtheta)\}$ is the internally transitive invariant sets of the ODE $\dot{\vw}(t)=g(\vtheta(t),\vw(t))$, $\dot{\vtheta}(t)=0$.
	\end{lemma}
	
	\begin{proof}[Proof of \lemref{lem: trans_invar}]

			Rewrite the generator update as 
			
			\begin{equation}
			{\vtheta}_{n+1}= {\vtheta}_{n}+a(n)[\epsilon_n+{\mM}_{n}^{(\vtheta')}]
			\end{equation}
			where $\epsilon_n:=\frac{b(n)}{a(n)}{h}({\vtheta}_{n},{\vw}_{n})$ and ${\mM}_{n}^{(\vtheta')}:=\frac{b(n)}{a(n)}{\mM}_{n}^{(\vtheta)}$ for $n\geq 0$. From the third Extension in Section 2.2 of \citep{borkar2009stochastic}, $\{\vv_n\}$ should converge to an internally chain transitive invariant set of $\dot{\vv(t)}=0$. Considering this, the proof follows directly from \thmref{thm: single_converge}.
	\end{proof}

	\begin{theorem}
	\label{thm: two_time}
		[Theorem 2, Chapter 6, \citep{borkar2009stochastic}] $(\vw_n,\vtheta_n)\rightarrow (\lambda(\vtheta^*), \vtheta^*)$  almost surely.
	\end{theorem}

\section{Proof of Results in Section \ref{sec:conv_analy}}\label{app: proof}
\begin{proof}	[Proof of \lemref{lem: bound_var}]
		\begin{equation}
		\label{eq: lem1_1}
		\begin{split}
		\E ||\vw^i_n-\vv_n||+\E||\vtheta^i_n-\vphi_n||&\leq \E||\vw^i_{n-1}-\vv_{n-1}-a(n-1)[\tilde{\vg}_i(\vw^i_{n-1}, \vtheta^i_{n-1})-\vg(\vv_{n-1}, \vphi_{n-1})]||\\
		& + \E||\vtheta^i_{n-1}-\vphi_{n-1}-b(n-1)[\tilde{\vh}_i(\vw^i_{n-1}, \vtheta^i_{n-1})-\vh(\vv_{n-1}, \vphi_{n-1})]||\\
		&\leq \E||\vw^i_{n-1}-\vv_{n-1}|| + \E||\vtheta^i_{n-1}-\vphi_{n-1}||\\
		&+a(n-1)\E||\tilde{\vg}_i(\vw^i_{n-1}, \vtheta^i_{n-1})-\vg(\vv_{n-1},\vphi_{n-1})||\\
		&+b(n-1)\E||\tilde{\vh}_i(\vw^i_{n-1}, \vtheta^i_{n-1})-\vh(\vv_{n-1},\vphi_{n-1})||
		\end{split}
		\end{equation}	
		The latter part on the right hand side can be written as
		\begin{equation}
		\label{eq: lem1_2}
		\begin{split}
		&a(n-1)\E||\tilde{\vg}_i(\vw^i_{n-1}, \vtheta^i_{n-1})-\vg(\vv_{n-1},\vphi_{n-1})||
		+b(n-1)\E||\tilde{\vh}_i(\vw^i_{n-1}, \vtheta^i_{n-1})-\vh(\vv_{n-1},\vphi_{n-1})||\\
		&=a(n-1)\E||[\tilde{\vg}_i(\vw^i_{n-1}, \vtheta^i_{n-1})-\vg_i(\vw^i_{n-1}, \vtheta^i_{n-1})+\vg_i(\vw^i_{n-1}, \vtheta^i_{n-1})-\vg_i(\vv_{n-1},\vphi_{n-1})\\
		&+\vg_i(\vv_{n-1},\vphi_{n-1})-\vg(\vv_{n-1},\vphi_{n-1})]||+b(n-1)\E||[\tilde{\vh}_i(\vw^i_{n-1}, \vtheta^i_{n-1})-\vh_i(\vw^i_{n-1}, \vtheta^i_{n-1})\\
		&+\vh_i(\vw^i_{n-1}, \vtheta^i_{n-1})-\vh_i(\vv_{n-1},\vphi_{n-1})+\vh_i(\vv_{n-1},\vphi_{n-1})-\vh(\vv_{n-1},\vphi_{n-1})]||\\
		&\leq a(n-1)(\sigma_g+\mu_g)+b(n-1)\sigma_h + [a(n-1)+b(n-1)]L[\E||\vw^i_{n-1}-\vv_{n-1}||+\E||\vtheta^i_{n-1}-\vphi_{n-1}||]\\
		&=a(n-1)[\sigma_g+\mu_g+\sigma_h+ 2L(\E||\vw^i_{n-1}-\vv_{n-1}|| + \E||\vtheta^i_{n-1}-\vphi_{n-1}||)]
		\end{split}
		\end{equation}	
		Here $\sigma_g$ and $\sigma_h$ are bounds for the variances for discriminator and generator respectively, and $\mu_g$ is the bound for gradient divergence for discriminator (Assumption (\textbf{A5})). The above inequality follows from Assumption (\textbf{A1}) and  $\vh_i(\vv_{n-1},\vphi_{n-1})=\vh(\vv_{n-1},\vphi_{n-1})$ which holds because the fake data is generated based on parameters of the generator. Also, the last equality holds because $a(n-1)=b(n-1)$ in equal time-scale updates.
		
		Considering (\ref{eq: lem1_1}) and (\ref{eq: lem1_2}), we have
		\begin{equation}
		\begin{split}
		\E ||\vw^i_n-\vv_n||&+\E||\vtheta^i_n-\vphi_n||\\
		&\leq (1+2a(n-1)L)(\E||\vw^i_{n-1}-\vv_{n-1}|| + \E||\vtheta^i_{n-1}-\vphi_{n-1}||)+a(n-1)(\sigma_g+\mu_g+\sigma_h)\\
		&\leq \frac{\sigma_g+\mu_g+\sigma_h}{2L}[(1+2a(n-1)L)^{n\, \text{mod}\, K}-1]
		\end{split}
		\end{equation}
		The last inequality can be holds by induction over $n$ and a mild assumption that the learning rate is unchanged within the same synchronization interval. 
	\end{proof}

\begin{proof}[Proof of \lemref{lem: bound_grad_div}]
		\begin{equation}
		\begin{split}
		&\E||{\vw}_{n}-\vv_{n}||+\E||{\vtheta}_{n}-\vphi_{n}||\\
		&\leq \E||{\vw}_{n-1}-\vv_{n-1}+a(n-1)\sum_i p_i [\tilde{\vg}_i(\vw^i_{n-1}, \vtheta^i_{n-1})-\vg_i(\vv_{n-1},\vphi_{n-1})]||\\
		&+ \E||{\vtheta}_{n-1}-\vphi_{n-1}+b(n-1)\sum_i p_i [\tilde{\vh}_i(\vw^i_{n-1}, \vtheta^i_{n-1})-\vh_i(\vv_{n-1},\vphi_{n-1})]||
		\end{split}
		\end{equation}
		
		We have
		\begin{equation}
		\begin{split}
		&\E||\tilde{\vg}_i(\vw^i_{n-1}, \vtheta^i_{n-1})-\vg_i(\vv_{n-1},\vphi_{n-1})||+\E||\tilde{\vh}_i(\vw^i_{n-1}, \vtheta^i_{n-1})-\vh_i(\vv_{n-1},\vphi_{n-1})||\\
		&=\E||\tilde{\vg}_i(\vw^i_{n-1}, \vtheta^i_{n-1})-\vg_i(\vw^i_{n-1}, \vtheta^i_{n-1})+\vg_i(\vw^i_{n-1}, \vtheta^i_{n-1})-\vg_i(\vv_{n-1},\vphi_{n-1})||\\
		&+\E||\tilde{\vh}_i(\vw^i_{n-1}, \vtheta^i_{n-1})-\vh_i(\vw^i_{n-1}, \vtheta^i_{n-1})+\vh_i(\vw^i_{n-1}, \vtheta^i_{n-1})-\vh_i(\vv_{n-1},\vphi_{n-1})||\\
		&\leq \sigma_g+\sigma_h+2L[\E||\vw^i_{n-1}-\vv_{n-1}||+\E||\vtheta^i_{n-1}-\vphi_{n-1}||]\\
		&\leq \sigma_g+\sigma_h+2L r_1(n-1)
		\end{split}
		\end{equation}
		where the last inequality follows from \lemref{lem: bound_var}. 
		
		Consequently,
		\begin{equation}
		\begin{split}
		&\E||{\vw}_{n}-\vv_{n}||+\E||{\vtheta}_{n}-\vphi_{n}|| \\
		&\leq \E||{\vw}_{n-1}-\vv_{n-1}||+\E||{\vtheta}_{n-1}-\vphi_{n-1}||+a(n-1) [\sigma_g+\sigma_h+2L r_1(n-1)]\\
		&=\E||{\vw}_{n-1}-\vv_{n-1}||+\E||{\vtheta}_{n-1}-\vphi_{n-1}||\\
		&+a(n-1)((\sigma_g+\sigma_h+\mu_g)[(1+2a(n-1)L)^{n\, \text{mod}\, K}-1]+\sigma_g+\sigma_h)\\
		&\leq a(n-1)(\sigma_g+\sigma_h+\mu_g)\sum_{j=1}^K (1+2a(n-1)L)^{j-1}-a(n-1)\mu_gK\\
		&=a(n-1)(\sigma_g+\sigma_h+\mu_g) \frac{(1+2a(n-1)L)^{K}-1}{2a(n-1)L}-a(n-1)\mu_gK\\
		&=\frac{(\sigma_g+\sigma_h+\mu_g)}{2L}[(1+2a(n-1)L)^{K}-1]-a(n-1)\mu_gK
		\end{split}
		\end{equation}
	\end{proof}
	
\begin{proof}[Proof of \thmref{thm: equal_traj}]
The proof is modified from Lemma 1 in Chapter 2 of \citep{borkar2009stochastic}.
		We shall only prove the first claim, as the arguments for proving the second claim are completely analogous. 
		Define 
		\begin{equation}
		\zeta^\vw_n=\sum_{m=0}^{n-1}a(m)\mM_{m}^{(\vw)}
		\end{equation}
		\begin{equation}
		\zeta^\vtheta_n=\sum_{m=0}^{n-1}b(m)\mM_{m}^{(\vtheta)}
		\end{equation}
		 Denote $\delta^{\vw}_{n,n+m}=\zeta^{\vw}_{n+m}-\zeta^{\vw}_n$, $\delta^{\vtheta}_{n,n+m}=\zeta^{\vtheta}_{n+m}-\zeta^{\vtheta}_n$. Let $\delta_{n,n+m} = (\delta^{\vw}_{n,n+m},\delta^{\vtheta}_{n,n+m})^\top$. Let  $t'(n)=t(nK)$ and $[t]=\max\{t(k): t(k)\leq t\}$. (Note that $t$ is overloaded here both as a function and a variable). By construction,
		\begin{equation}
		\begin{split}
		&\vc^{\vw}_0 = \bar{\vw}(t'(n))\\
		&\vc^{\vw}_k = \bar{\vw}(t'(n))  +\sum_{j=0}^{k-1} a(nK+j)\vg(\vc^{\vw}_{j}, \vc^{\vtheta}_{j}) + \delta^{\vw}_{nK,nK+k}
		\end{split}
		\end{equation}
		
		\begin{equation}
		\begin{split}
		&\vc^{\vtheta}_0 = \bar{\vtheta}(t'(n))\\
		&\vc^{\vtheta}_k = \bar{\vtheta}(t'(n))  +\sum_{j=0}^{k-1} b(nK+j)\vh(\vc^{\vw}_{j}, \vc^{\vtheta}_{j}) + \delta^{\vtheta}_{nK,nK+j}
		\end{split}
		\label{eq: recons_seq1t}
		\end{equation}
		Furthermore, denote
		\begin{equation}
		\vc_k=(\vc^{\vw}_k, \vc^{\vtheta}_k)^\top=\bar{\vz}(t'(n))+\sum_{j=0}^{k-1}a(nK+j)\vf(\vc_j)+\delta_{nK,nK+j}
		\label{eq: recons_seq1}
		\end{equation}
		
		Also by construction,
		
		\begin{equation}
		\begin{split}
		\vd^{\vw}_0 &= \bar{\vw}(t'(n))\\
		\vd^{\vw}_k &=\bar{\vw}(t'(n))+\sum_{j=0}^{k-1} a(nK+j)\vg(\vd^{\vw}_{j}, \vd^{\vtheta}_{j})\\
		&+\int_{t(nK)}^{t(nK+k)}(\vg( \vw^{t'(n)}(y),\vtheta^{t'(n)}(y))-\vg( \vw^{t'(n)}([y]),  \vtheta^{t'(n)}([y])))dy
		\end{split}
		\end{equation}
		
		\begin{equation}
		\begin{split}
		\vd^{\vtheta}_0 &= \bar{\vtheta}(t'(n))\\
		\vd^{\vtheta}_k &=\bar{\vtheta}(t'(n))+\sum_{j=0}^{k-1} b(nK+j)\vh(\vd^{\vw}_{j}, \vd^{\vtheta}_{j})\\
		&+\int_{t(nK)}^{t(nK+k)}(\vh( \vw^{t(n)}(y),\vtheta^{t(n)}(y))-\vh( \vw^{t(n)}([[y]]),  \vtheta^{t(n)}([y])))dy
		\end{split}
		\label{eq: recons_seq2t}
		\end{equation}
		
		Further denote
		\begin{equation}
		\vd_k=(\vd^{\vw}_k, \vd^{\vtheta}_k)^\top
		\label{eq: recons_seq2}
		\end{equation}
		
		Let $C_0=\sup_n||\vw_n||<\infty$, $C_1=\sup_n||\vtheta_n||<\infty$ almost surely (Assumption (\textbf{A4})), let $L>0$ denote the Lipschitz constant of $\vg$ and $\vh$, and let $s\leq t\leq s+T$. Note that $||\vg(x)-\vg(0)||\leq L||x||$, and so $||\vg(x)||\leq ||\vg(0)||+L||x||$. Similar inequalities also hold for $\vh$. Since $\vw^s(t)=\bar{\vw}(s)+\int_s^t \vg(\vw^s(\tau), \vtheta^s(\tau))\vd\tau$ and $\vtheta^s(t)=\bar{\vtheta}(s)+\int_s^t \vh(\vw^s(\tau), \vtheta^s(\tau))\vd\tau$, for discriminator we have
		\begin{equation}
		\begin{split}
		||\vw^s(t)||&\leq ||\bar{\vw}(t)||+\int_s^t [||\vg(0,0)||+L||\vw^s(\tau)||+L||\vtheta^s(\tau)||]\vd\tau\\
		&\leq (C_0+||\vg(0)||T)+L\int_s^t(||\vw^s(\tau)||+||\vtheta^s(\tau)||)\vd\tau
		\end{split}
		\end{equation}
and for generator we have	\begin{equation}
		\begin{split}
		||\vtheta^s(t)||&\leq ||\bar{\vtheta}(t)||+\int_s^t [||\vh(0,0)||+L||\vw^s(\tau)||+L||\vtheta^s(\tau)||]\vd\tau\\
		&\leq (C_1+||\vh(0)||T)+L\int_s^t(||\vw^s(\tau)||+||\vtheta^s(\tau)||)\vd\tau
		\end{split}
		\end{equation}
		
		Adding the above two inequalities, we have
		\begin{equation}
		||\vw^s(t)||+||\vtheta^s(t)||\leq (C_0+C_1+||\vg(0)||T+||\vh(0)||T)+2L\int_s^t(||\vw^s(\tau)||+||\vtheta^s(\tau)||)\vd\tau
		\end{equation}
		
		By Gronwall's inequality, it follows that 
		\begin{equation}
		||\vw^s(t)||+||\vtheta^s(t)||\leq (C_0+C_1+||\vg(0)||T+||\vh(0)||T)e^{2LT}, s\leq t\leq s+T
		\end{equation}
		
		Thus, for all $s\leq t\leq s+T$
		\begin{equation}
		||\vg(\vw^s(t), \vtheta^s(t))||\leq C_T = ||\vg(0)||+L(C_0+C_1+||\vg(0)||T+||\vh(0)||T)e^{2LT}<\infty, a.s.
		\end{equation}
		\begin{equation}
		||\vh(\vw^s(t), \vtheta^s(t))||\leq C'_T = ||\vh(0)||+L(C_0+C_1+||\vg(0)||T+||\vh(0)||T)e^{2LT}<\infty, a.s.
		\end{equation}
		Now, if $0 \leq j< mK$ and $t\in (t'(nK+j),t'(nK+j+1)]$,
		\begin{equation}
		\begin{split}
		|| \vw^{t(n)}(t)- \vw^{t(n)}(t'(nK+j))|| &\leq ||\int_{t'(nK+j)}^t \vg(\vw^{t(n)}(s), \vtheta^{t(n)}(s))ds||\\
		&\leq C_T(t-t'(nK+j))\\
		&\leq C_Ta(nK+j) 
		\end{split}
		\end{equation}
		\begin{equation}
		\begin{split}
		|| \vtheta^{t(n)}(t)- \vtheta^{t(n)}(t'(nK+j)|| &\leq ||\int_{t'(nK+j)}^t \vh(\vw^{t(n)}(s), \vtheta^{t(n)}(s))ds||\\
		&\leq C'_T(t-t'(nK+j))\\
		&\leq C'_Ta(nK+j)
		\end{split}
		\end{equation}
		Thus,
		\begin{equation}
		\begin{split}
		&\int_{t(nK)}^{t(nK+k)}(\vg( \vw^{t'(n)}(t),\vtheta^{t’(n)}(t))-\vg( \vw^{t'(n)}([t]),\vtheta^{t’(n)}([t])))dt\\
		&\leq\int_{t(nK)}^{t(nK+k)}L(||\vw^{t'(n)}(t)-\vw^{t'(n)}([t]))||+||\vtheta^{t'(n)}(t)-\vtheta^{t'(n)}([t]))||)dt\\
		&=L \sum_{j=0}^{k-1} \int_{t(nK+j)}^{t(nK+j+1)}(||\vw^{t'(n)}(t)-\vw^{t’(n)}(t(nK+j)))|| +|| \vtheta^{t'(n)}(t)- \vtheta^{t’(n)}(t(nK+j)||)dt\\
		&\leq (C_T+C'_T) L\sum_{j=0}^{k-1} a(nK+j) ^2 \\
		&\leq (C_T+C'_T) L\sum_{j=0}^{\infty} a(nK+j) ^2  \xrightarrow{n \uparrow \infty }0, a.s.
		\end{split}
		\end{equation}
		
		Similarly, we have
		\begin{equation}
		\begin{split}
		&\int_{t(nK)}^{t(nK+k)}(\vh( \vw^{t'(n)}(t),\vtheta^{t’(n)}(t))-\vh( \vw^{t'(n)}([t]),  \vtheta^{t’(n)}([t])))dt\\
		&\leq (C_T+C'_T) L\sum_{j=0}^{\infty} a(nK+j) ^2   \xrightarrow{n \uparrow \infty }0, a.s.
		\end{split}
		\end{equation}
		

		Also by \lemref{lem: bound_var}, \lemref{lem: bound_grad_div}, we have
		\begin{equation}
		\sup_{k\geq 0}||\delta_{nK,nK+k}||\xrightarrow{n \uparrow \infty }0, a.s.
		\end{equation}
		
		Subtracting (\ref{eq: recons_seq2}) from (\ref{eq: recons_seq1}) and taking norms, we have
		\begin{equation}
		\begin{split}
		||\vc_k-\vd_k|| &\leq L \sum_{j=0}^{k-1} a(nK+j)||\vc_j-\vd_{j}||\\
		&+2(C_T+C'_T) L\sum_{j=0}^{\infty} a(nK+j) ^2+\sup_{k\geq 0}||\delta_{nK,nK+k}||, a.s.
		\end{split}
		\end{equation}
		Define $K_{T,n}=2(C_T+C'_T) L\sum_{j=0}^{\infty} a(nK+j) ^2+\sup_{k\geq 0}||\delta_{nK,nK+k}||$. Note that $K_{T,n}\rightarrow 0$ almost surely as $n\rightarrow \infty$ (by Assumption (\textbf{A2})). Also, let $\vy_k=||\vc_k-\vd_k||$ and $q_j=a(nK+j)$. Thus, the above inequality becomes
		\begin{equation}
		\vy_k\leq K_{T,n}+L\sum_{j=0}^{k-1}q_j\vy_j
		\end{equation}
		Note that $y_0=0$ and $\sum_{j=0}^{k-1}q_j\leq KT$. The discrete Gronwall lemma tells that 
		\begin{equation}
		\label{eq: bound_lin_interpo}
		\sup_k \vy_k\leq  K_{T,n} e^{L KT}
		\end{equation}
		
		Since $K_{T,n}\rightarrow 0$ almost surely as $n\rightarrow \infty$, the original statement of the Theorem follows directly from linear interpolation of (\ref{eq: bound_lin_interpo}).
	\end{proof}

\section{Toy Examples in Section \ref{sec:toy}: 2D System, Mixed Gaussian and Swiss Roll}\label{app: toy}

\citep{nagarajan2017gradient} experiments convergence of GANs using a 2D example. We extend this experiment for FedGAN setup. We consider a 2D system where both the true distribution for $x$ and the latent distribution for $z$ are uniform over $[-1, 1]$, the discriminator is $D(x) = \psi x^2$, and the generator is $G(z) = \theta z$. For the federated experiment, we assume there are $B=5$ agents by partitioning the data domain into $5$ equal segments with each agent's data coming from one (e.g. Agent $1$'s true distribution will have be uniform over $[-1,-0.6)$). We set synchronization interval $K=1,5, 20, 50$.
	
	The $(\theta, \phi)$ trajectory is plotted in Figure \ref{fig:2d}. In all cases FedGAN converges to the same point of $(1, 0)$. At this point, the generator generated data with uniform distribution between $[-1,1]$, and the discriminator classifies all data in one class (fails to discriminate). This experiment shows robustness of the FedGAN to increasing parameter $K$ and reducing communications. In other words, reducing the communications has low impact on the algorithm result. Note that while $(1,0)$ is the convergence point of the centralized GAN, it may not be the case in general as we observe in later experiments.

	\begin{figure}[ht]
		\begin{subfigure}[t]{0.5\columnwidth}
		
			\centerline{\includegraphics[width=\textwidth]{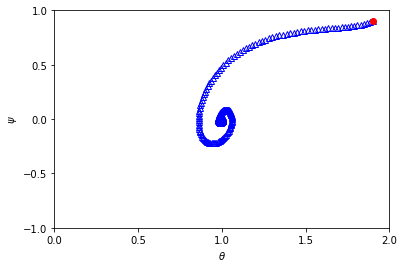}}
			\caption{K=1}
		\end{subfigure}\hfill
		\begin{subfigure}[t]{0.5\columnwidth}
		
			\centerline{\includegraphics[width=\textwidth]{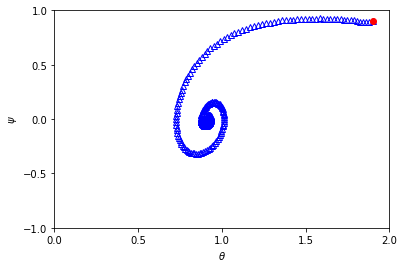}}
			\caption{K=5}
		\end{subfigure}\hfill
		\begin{subfigure}[t]{0.5\columnwidth}
		
			\centerline{\includegraphics[width=\textwidth]{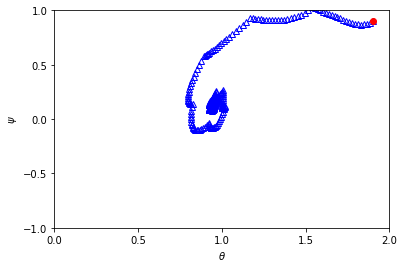}}
			\caption{K=20}
		\end{subfigure}\hfill
		\begin{subfigure}[t]{0.5\columnwidth}
		
			\centerline{\includegraphics[width=\textwidth]{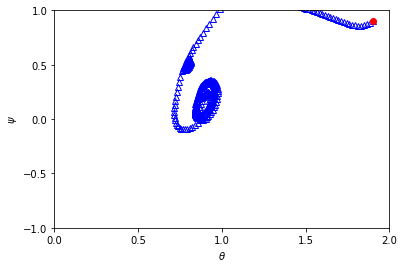}}
			\caption{K=50}
		\end{subfigure}\hfill
		
		\caption{FedGAN parameters trajectory in a 2D experiment, generator $\theta$ and discriminator $\phi$ with $B=5$ agents and synchronization intervals $K=1,5,20, 50$. The red dot is the initial values.}
		\label{fig:2d}
	\end{figure}

    \citep{Metz2016UnrolledGA} and \citep{gulrajani2017improved} train GAN on the popular 2D mixture of eight Gaussians arranged in a circle, and Swiss roll dataset. We extend both experiments to FedGAN by dividing the data into $B=4$ agents (each owning 2 Gaussians or different but equal-sized part of the roll). We set $K=5$. The neural network structure used is the same as that in \citep{kodali2017convergence}.
    
    In Figure \ref{fig:mix_gau} and \ref{fig:swiss_roll},  for mixed Gaussian and Swiss role FedGAN experiments, the orange points represent the real data and the green points represent the generated data which are almost coinciding for $N=15000$ and $N=27000$ iterations correspondingly. These ensures high performance of FedGAN in generating fake samples representing the pooled data.

	\begin{figure}[ht]
		
		\begin{subfigure}[t]{0.5\columnwidth}
			\centerline{\includegraphics[width=\textwidth]{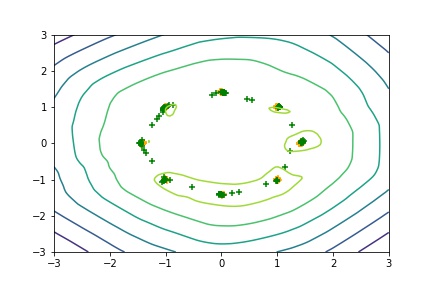}}
			\caption{iter=100}
		\end{subfigure}\hfill
		\begin{subfigure}[t]{0.5\columnwidth}
			\centerline{\includegraphics[width=\textwidth]{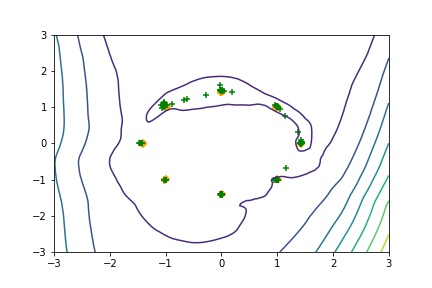}}
			\caption{iter=5000}
		\end{subfigure}\hfill
		\begin{subfigure}[t]{0.5\columnwidth}
			\centerline{\includegraphics[width=\textwidth]{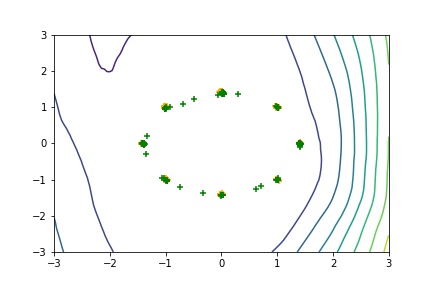}}
			\caption{iter=10000}
		\end{subfigure}\hfill
		\begin{subfigure}[t]{0.5\columnwidth}
			\centerline{\includegraphics[width=\textwidth]{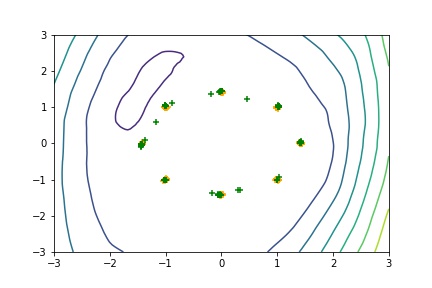}}
			\caption{iter=15000}
		\end{subfigure}\hfill
		\caption{Mixed Gaussians real data, and FedGAN generated data  with $B=4$ agents and synchronization interval $K=5$.}
		\label{fig:mix_gau}
	\end{figure}
	

	\begin{figure}[ht]
		
		\begin{subfigure}[t]{0.5\columnwidth}
			\centerline{\includegraphics[width=\textwidth]{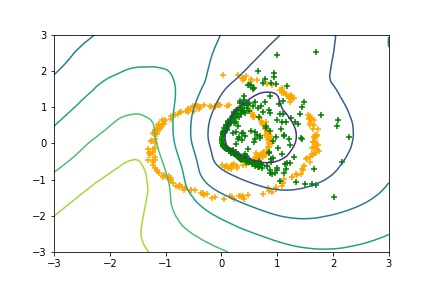}}
			\caption{iter=100}
		\end{subfigure}\hfill
		\begin{subfigure}[t]{0.5\columnwidth}
			\centerline{\includegraphics[width=\textwidth]{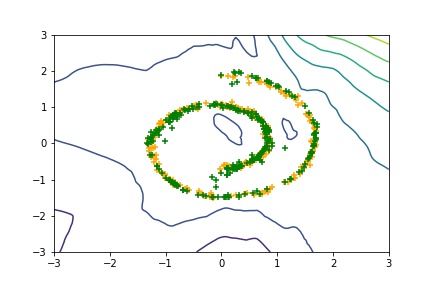}}
			\caption{iter=9000}
		\end{subfigure}\hfill
		\begin{subfigure}[t]{0.5\columnwidth}
			\centerline{\includegraphics[width=\textwidth]{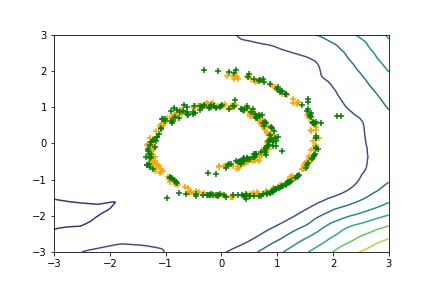}}
			\caption{iter=18000}
		\end{subfigure}\hfill
		\begin{subfigure}[t]{0.5\columnwidth}
			\centerline{\includegraphics[width=\textwidth]{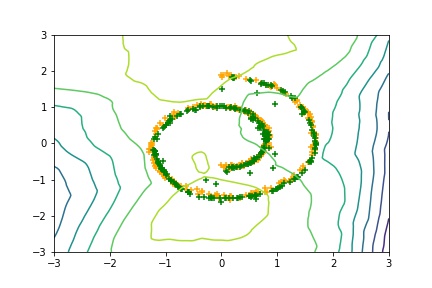}}
			\caption{iter=27000}
		\end{subfigure}\hfill
		\caption{Swiss role real data, and FedGAN generated data  with $B=4$ agents and synchronization interval $K=5$.}
		\label{fig:swiss_roll}
	\end{figure}

\section{Hyperparameters and Generated Images for Experiments in Section  \ref{sec:exp_img} and  \ref{sec:exp_energy} } \label{app: net_hyper}	
For our proposed FedGAN algorithm, the hyperparamters include those of local generators and discriminators.
For the experiments in Section \ref{sec:exp_img}, Table \ref{tab:cifar10} lists the hyperparameters used for CIFAR-10, and  Table \ref{tab:celebA} shows those for CelebA. The hyperparameters used for MNIST are the same as CIFAR-10 except that the input dimension of discriminator and the output dimension of generator (number of channels of greyscale images) are both equal to $1$. More samples of generated images for CIFAR-10 and CelebA experiments are shown in Figure \ref{fig:cifar_matrix} and \ref{fig:celebA_matrix} respectively.

The hyperparameters for the PG\&E and EV time series data in Section \ref{sec:exp_energy} are shown in Table \ref{tab:time_series}.  

\begin{table}[htbp]
\caption{CIFAR-10 hyperparameters. The learning rates for generator and discriminator are both equal to the same value, across all cases in Figure \ref{cifar10_fid}. BN stands for batch normalization.}
\centering
\begin{tabular}{llllll}
\hline 
Operation& Kernel &Strides &Feature maps & BN? & Nonlinearity\\
\hline 
$G(z)$ - 62 $\times$ 1 $\times$ 1 input  & & & & & \\
Linear & N/A & N/A&1024 & Y & ReLU\\
Linear & N/A & N/A&128 & Y & ReLU\\
Transposed Convolution &4$\times$ 4 & 2$\times$ 2 &64 & Y & ReLU\\
Transposed Convolution &4$\times$ 4 & 2$\times$ 2 &3 & N & Tanh\\
$D(x)$ - 32 $\times$ 32 $\times$ 3 input  & & & & & \\
Convolution &4$\times$ 4 & 2$\times$ 2 &64 & N & Leaky ReLU\\
Convolution &4$\times$ 4 & 2$\times$ 2 &128 & Y & Leaky ReLU\\
Linear & N/A & N/A&1024 & Y &Leaky ReLU\\
Linear (binary) & N/A & N/A&1 & Y &Sigmoid\\
Linear (classify)& N/A & N/A&10 & Y &-\\
Generator Optimizer &\multicolumn{5}{l}{Adam($\alpha=0.001, \beta_1=0.5, \beta_2=0.999$)}\\
Discriminator Optimizer &\multicolumn{5}{l}{Adam($\alpha=0.001, \beta_1=0.5, \beta_2=0.999$)}\\
Batch size &\multicolumn{5}{l}{64}\\
Leaky ReLU slope &\multicolumn{5}{l}{0.2}\\
\hline
\end{tabular}
\label{tab:cifar10}
\end{table}

\begin{figure}[ht]
		
		\begin{subfigure}[t]{0.5\columnwidth}
			\centerline{\includegraphics[width=\textwidth]{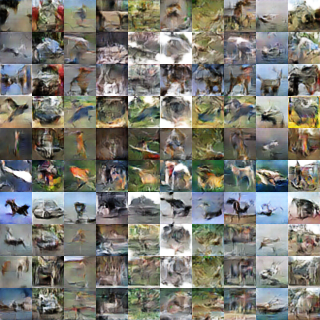}}
			\caption{$K=10$}
		\end{subfigure}\hfill
		\begin{subfigure}[t]{0.5\columnwidth}
			\centerline{\includegraphics[width=\textwidth]{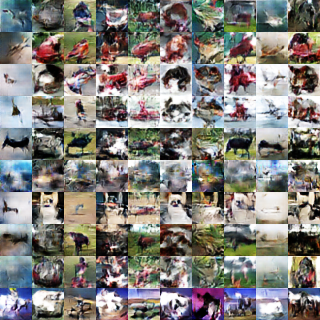}}
			\caption{$K=20$}
		\end{subfigure}\hfill
		\begin{subfigure}[t]{0.5\columnwidth}
			\centerline{\includegraphics[width=\textwidth]{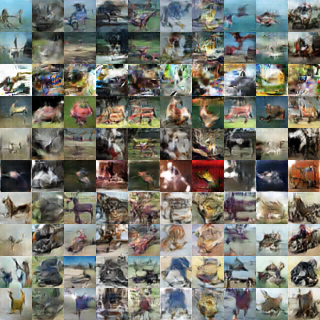}}
			\caption{$K=100$}
		\end{subfigure}\hfill
		\begin{subfigure}[t]{0.5\columnwidth}
			\centerline{\includegraphics[width=\textwidth]{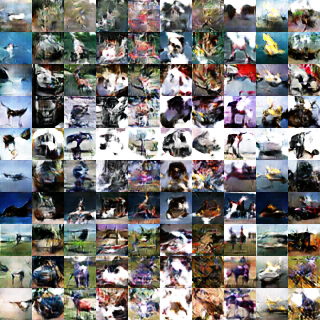}}
			\caption{$K=500$}
		\end{subfigure}\hfill
		\begin{subfigure}[t]{0.5\columnwidth}
			\centerline{\includegraphics[width=\textwidth]{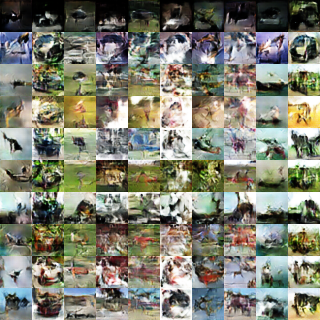}}
			\caption{$K=3000$}
		\end{subfigure}\hfill
		\begin{subfigure}[t]{0.5\columnwidth}
			\centerline{\includegraphics[width=\textwidth]{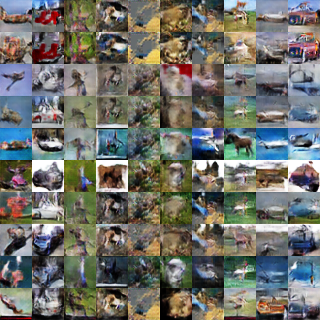}}
			\caption{Distributed GAN}
		\end{subfigure}\hfill
		\caption{Generated images for CIFAR-10 with $B=5$, $K=10,20,100,500,3000$ and distributed GAN, $N=30000$ iteration.}
		\label{fig:cifar_matrix}
	\end{figure}

\begin{table}[htbp]
\caption{CelebA hyperparameters. The learning rates for generator and discriminator are for $K=10,20,50,100,200$ and distributed GAN respectively as in Figure \ref{fig:celebA_fid}. BN stands for batch normalization.}
\centering
\begin{tabular}{llllll}
\hline 
Operation& Kernel &Strides &Feature maps & BN? & Nonlinearity\\
\hline 
$G(z)$ - 62 $\times$ 1 $\times$ 1 input  & & & & & \\
Linear & N/A & N/A&640 & Y & ReLU\\
Transposed Convolution &4$\times$ 4 & 2$\times$ 2 &320 & Y & ReLU\\
Transposed Convolution &4$\times$ 4 & 2$\times$ 2 &160 & Y & ReLU\\
Transposed Convolution &4$\times$ 4 & 2$\times$ 2 &80 & Y & ReLU\\
Transposed Convolution &4$\times$ 4 & 2$\times$ 2 &3 & N & Tanh\\
$D(x)$ - 48 $\times$ 48 $\times$ 3 input  & & & & & \\
Convolution &4$\times$ 4 & 2$\times$ 2 &80 & N & Leaky ReLU\\
Convolution &4$\times$ 4 & 2$\times$ 2 &160 & Y & Leaky ReLU\\
Convolution &4$\times$ 4 & 2$\times$ 2 &320 & Y & Leaky ReLU\\
Convolution &4$\times$ 4 & 2$\times$ 2 &640 & Y & Leaky ReLU\\
Linear (binary) & N/A & N/A&1 & Y &-\\
Linear (classify)& N/A & N/A&16 & Y &Sigmoid\\
Generator Optimizer &\multicolumn{5}{l}{Adam($\alpha=[1,1,1,1,1,1]\times 10^{-4}, \beta_1=0.5, \beta_2=0.999$)}\\
Discriminator Optimizer &\multicolumn{5}{l}{Adam($\alpha=[2,2,2,2,2,1]\times 10^{-4}, \beta_1=0.5, \beta_2=0.999$)}\\
Batch size &\multicolumn{5}{l}{128}\\
Leaky ReLU slope &\multicolumn{5}{l}{0.2}\\
\hline
\end{tabular}
\label{tab:celebA}
\end{table}

\begin{figure}[ht]
		
		\begin{subfigure}[t]{0.5\columnwidth}
			\centerline{\includegraphics[width=\textwidth]{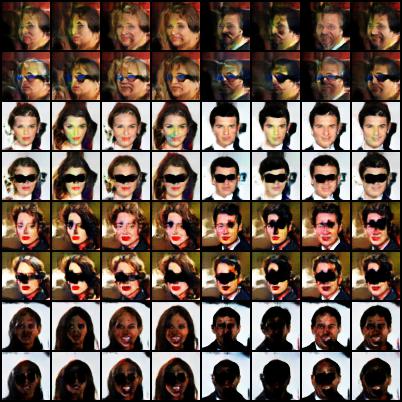}}
			\caption{$K=10$}
		\end{subfigure}\hfill
		\begin{subfigure}[t]{0.5\columnwidth}
			\centerline{\includegraphics[width=\textwidth]{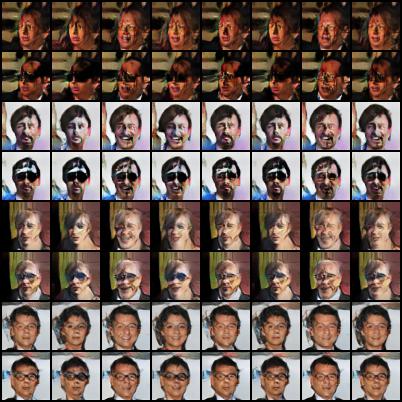}}
			\caption{$K=20$}
		\end{subfigure}\hfill
		\begin{subfigure}[t]{0.5\columnwidth}
			\centerline{\includegraphics[width=\textwidth]{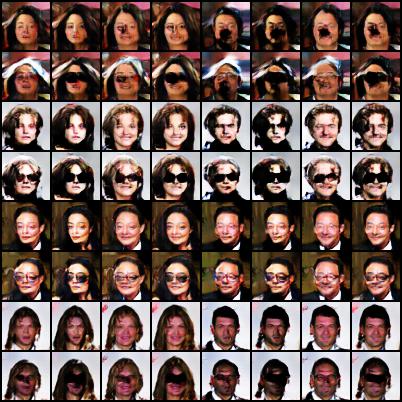}}
			\caption{$K=50$}
		\end{subfigure}\hfill
		\begin{subfigure}[t]{0.5\columnwidth}
			\centerline{\includegraphics[width=\textwidth]{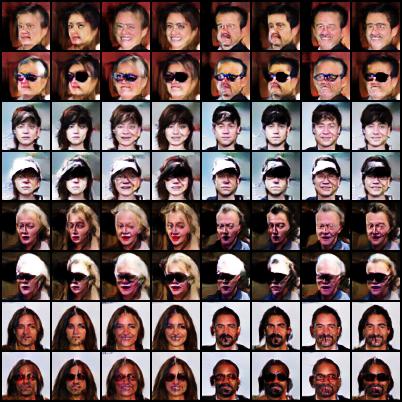}}
			\caption{$K=100$}
		\end{subfigure}\hfill
		\begin{subfigure}[t]{0.5\columnwidth}
			\centerline{\includegraphics[width=\textwidth]{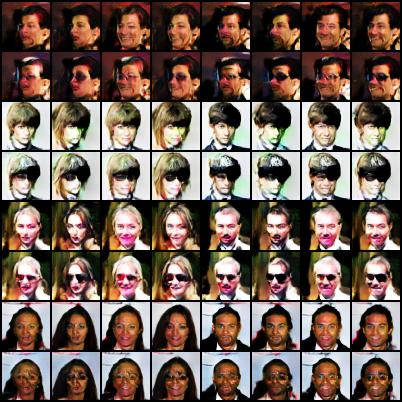}}
			\caption{$K=200$}
		\end{subfigure}\hfill
		\begin{subfigure}[t]{0.5\columnwidth}
			\centerline{\includegraphics[width=\textwidth]{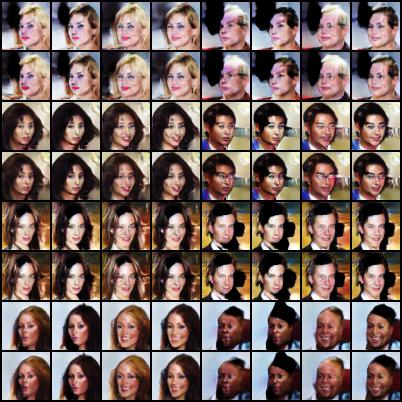}}
			\caption{Distributed GAN}
		\end{subfigure}\hfill
		\caption{Generated images for CelebA with $B=5$, $K=10,20,50,100,200$ and distributed GAN, $N=47500$ iteration.}
		\label{fig:celebA_matrix}
	\end{figure}

\begin{table}[htbp]
\caption{Time series data, PG\&E and EV, hyperparameters. }
\centering
\begin{tabular}{lllll}
\hline 
Operation& Kernel &Strides &Feature maps & Nonlinearity\\
\hline 
$G(z)$ - (label dimension+1) $\times$ 24  input  & & & & \\
1d convolution & 5 & 1 & 64  & -\\
1d convolution & 5 & 1 & 64  & ReLU\\
\multicolumn{5}{l}{Repeat the above 1d convolution 8 times}\\
1d convolution & 1 & 1 & 1  & -\\
$D(x)$ -  (label dimension+1) $\times$ 24  input  & & & & \\
1d convolution & 5 & 1 & 64 &  -\\
1d convolution & 5 & 1 & 64 &  ReLU\\
\multicolumn{5}{l}{Repeat the above 1d convolution 8 times}\\
Linear & 1 & 1 & 1  & -\\
Generator Optimizer &\multicolumn{4}{l}{Adam($\alpha=0.0001\times 10^{-4}, \beta_1=0.5, \beta_2=0.999$)}\\
Discriminator Optimizer &\multicolumn{4}{l}{Adam($\alpha=0.0004\times 10^{-4}, \beta_1=0.5, \beta_2=0.999$)}\\
Batch size &\multicolumn{4}{l}{256}\\
\hline
\end{tabular}
\label{tab:time_series}
\end{table}	
	
\section{Supplementary EV Data Description in Section \ref{sec:exp_energy}}\label{app: data}

For the EV dataset, we observe the following characteristics at the station level: station ID, connector type (e.g., J1772), POI category (e.g., workplace, retail, municipal), 	POI subcategory (e.g., commercial, high-tech), station zip code, and max power (e.g., 6.6kW, 24kW, 50kW). Also, at the user level, we observe: driver ID, home zip code, vehicle make, vehicle model, vehicle model year, battery capacity, and EV type (e.g., plugin, hybrid). For the charging profile generating problem considered in Section \ref{sec:exp_energy}, we are using POI category/subcategory, max power, battery capacity, month, and day of week as labels for CGAN. 

We split the charging dataset to $5$ agents based on the category of charging stations. Therefore, the data distribution across the agents is non-iid; each agent has a different distribution of charging profiles. Figure \ref{fig: ev_distri} shows this non-iid data distribution by comparing two charging stations (each belonging to a different agent), a high-tech workplace and a shopping center. The plots are all the charging sessions on Tuesdays for these two stations. Most charging sessions in the high-tech workplace happen during the day time which is consistent with people working close by. On the other hand, charging sessions in the shopping center last till midnight which is probably when it closes. 

\begin{figure}[htbp]
		\begin{subfigure}[b]{0.5\textwidth}
			\centerline{\includegraphics[scale=0.4]{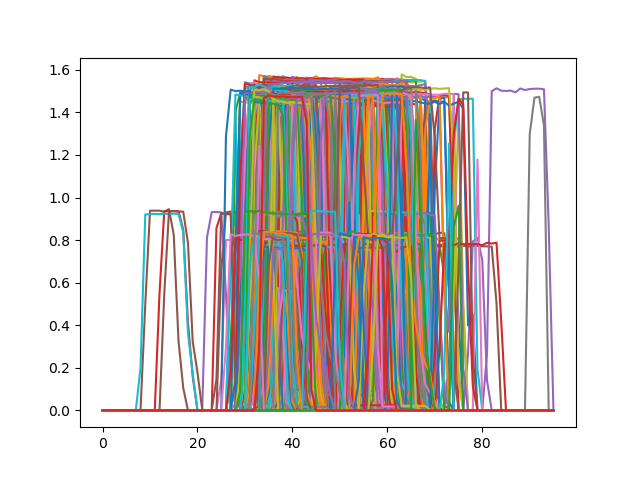}}
			\caption{High-tech workplace}
		\end{subfigure}
		\begin{subfigure}[b]{0.5\textwidth}
			\centerline{\includegraphics[scale=0.4]{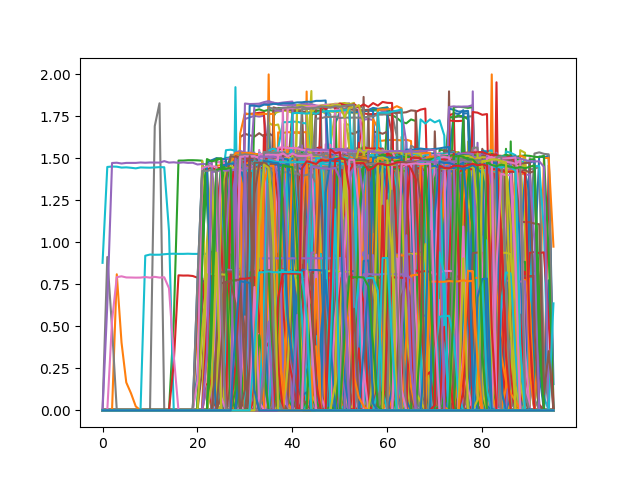}}
			\caption{Shopping center}
		\end{subfigure}
		\caption{Charging profiles on Tuesday for two stations from different categories.}
		\label{fig: ev_distri}
	\end{figure}

\end{document}